\documentclass[preprint]{article}
\makeatletter
\@ifundefined{abx@aux@refcontext}{\providecommand{\abx@aux@refcontext}[1]{}}{}
\makeatother

\usepackage{icml2026}
\usepackage{natbib}
\usepackage{threeparttable}
\usepackage{multirow}
\usepackage{color}
\usepackage{diagbox}
\usepackage{amsmath,amssymb,amsthm}
\newtheorem{theorem}{Theorem}
\newtheorem{lemma}{Lemma}
\newtheorem{corollary}{Corollary}

\newcommand{\E}{\mathbb{E}}
\newcommand{\R}{\mathbb{R}}
\newcommand{\norm}[1]{\left\lVert #1 \right\rVert}
\usepackage{mathtools} 
\usepackage{enumitem}
\usepackage{caption}
\usepackage{caption,subcaption}

\usepackage[utf8]{inputenc} 
\usepackage[T1]{fontenc}    
\usepackage{hyperref}       
\usepackage{url}         
\usepackage{booktabs}       
\usepackage{amsfonts}       
\usepackage{nicefrac}       
\usepackage{microtype}      
\usepackage{xcolor}         
\usepackage{wrapfig}
\usepackage[textsize=tiny]{todonotes}

\newcommand{\lmNum}[0]{5}

\definecolor{DarkGreen}{RGB}{1,100,32}
\definecolor{DarkRed}{RGB}{158,19,22}

\newcommand{\Answer}[2]{\noindent \textbf{Answer to #1:} #2}
\usepackage{tikz}
\definecolor{customblue}{RGB}{192, 208, 235}
\definecolor{customdarkblue}{RGB}{68, 114, 196}

\begin{document}

\twocolumn[
  \icmltitle{An Empirical Study and Theoretical Explanation on Task-Level Model-Merging Collapse}

  \begin{icmlauthorlist}
    \icmlauthor{Yuan Cao}{scs}
    \icmlauthor{Dezhi Ran$^*$$^\dagger$}{scs}
    \icmlauthor{Yuzhe Guo}{scs}
    \icmlauthor{Mengzhou Wu}{scs}
    \icmlauthor{Simin Chen}{cbu}
    \icmlauthor{Linyi Li}{sfu}
    \icmlauthor{Wei Yang}{utdallas}
    \icmlauthor{Tao Xie$^*$}{scs}
  \end{icmlauthorlist}

  \icmlaffiliation{scs}{Key Lab of HCST (PKU), MOE; SCS, Peking University, Beijing, China}
  \icmlaffiliation{cbu}{Computer Science Department at Columbia University, New York, USA}
  \icmlaffiliation{utdallas}{Department of Computer Science, University of Texas at Dallas, Richardson, USA}
  \icmlaffiliation{sfu}{School of Computing Science, Simon Fraser University, Burnaby, BC, Canada}

  \icmlcorrespondingauthor{Tao Xie}{taoxie@pku.edu.cn}
  \icmlcorrespondingauthor{Dezhi Ran}{dezhiran@pku.edu.cn}

\icmlkeywords{GUI Agents, Vision-Language Models, Synthetic Environment}
  \vskip 0.3in
]
\printAffiliationsAndNotice{\textsuperscript{*}Co-corresponding authors. \textsuperscript{$\dagger$}Project Leader.}


\begin{abstract}
Model merging unifies independently fine-tuned LLMs from the same base, enabling reuse and integration of parallel development efforts without retraining.
However, in practice we observe that merging does not always succeed: certain combinations of task-specialist models suffer from catastrophic performance degradation after merging. We refer to this failure mode as  merging collapse. Intuitively, collapse arises when the learned representations or parameter adjustments for different tasks are fundamentally incompatible, so that merging forces destructive interference rather than synergy.
In this paper, we identify and characterize the phenomenon of task-level merging collapse, where certain task combinations consistently trigger huge performance degradation across all merging methods. 
Through extensive experiments and statistical analysis, we demonstrate that representational incompatibility between tasks is strongly correlated with merging collapse, while parameter-space conflict metrics show minimal correlation, challenging conventional wisdom in model merging literature.
We provide a theoretical explanation on this phenomenon through rate-distortion theory with a dimension-dependent bound, establishing fundamental limits on task mergeability regardless of methodology.

\end{abstract}

\section{Introduction}\label{sec::intro}
Large language models (LLMs)~\cite{brown2020language, bommasani2021opportunities,achiam2023gpt, grattafiori2024llama3herdmodels,yang2024qwen2,guo2025deepseek} have demonstrated remarkable success in numerous downstream tasks from natural language processing~\cite{wang2018glue,hendrycks2020measuring} to code generation~\cite{lu2021codexglue,chen2021evaluating,jain2024livecodebench,li2024infibench}.
While LLMs achieve strong general performance~\cite{achiam2023gpt, grattafiori2024llama3herdmodels,yang2024qwen2,guo2025deepseek}, adapting them to specific domains through fine-tuning remains computationally intensive~\cite{yadav2024ties}, particularly when multiple specialized variants are required for different tasks or applications~\cite{wolf2019huggingface,yadav2024ties}.
This challenge has spurred interests in parameter-efficient methods~\cite{hu2022lora,yadav2024ties} to combine and adapt pretrained models without exhaustive retraining.

Among emerging solutions, model merging~\cite{zhang2019malmem,wortsman2022model,ilharco2023editingmodelstaskarithmetic,yadav2024ties,yu2024language,lu2024twin} has shown particular promise as a computationally efficient approach to combine multiple fine-tuned LLMs derived from the same base model.
By merging model parameters directly, model merging aims to consolidate specialized capabilities into a single unified model while avoiding the costs of additional training. 
By directly combining model parameters, merging leverages the widely observed linear mode connectivity~\cite{frankle2020linear, zhou2024emergence} between fine-tuned variants to create multi-capability models. Task arithmetic methods~\cite{zhang2019malmem, wortsman2022model, ilharco2023editingmodelstaskarithmetic} and their recent extensions~\cite{yadav2024ties, yu2024language, lu2024twin} have shown particular effectiveness for specific task combinations. However, a critical question remains underexplored: What fundamental limitations govern which tasks can be successfully merged without model-merging collapse (defined in Section~\ref{sec:preliminary})?

In this paper, we identify and characterize the phenomenon of task-level merging collapse. 
Through comprehensive empirical analysis spanning five state-of-the-art merging methods~\cite{zhang2019malmem,wortsman2022model,ilharco2023editingmodelstaskarithmetic,yadav2024ties,yu2024language}, diverse LLM architectures~\cite{grattafiori2024llama3herdmodels,yang2024qwen2}, and fine-tuning approaches~\cite{howard2018universal,hu2022lora}, we observe a consistent and alarming pattern: model merging can catastrophically fail when combining certain task combinations, despite each individual model performing well in isolation.
Contrary to the prevailing focus on parameter conflicts~\cite{yadav2024ties}, our statistical analysis reveals that parameter-space metrics show minimal correlation.

To explain this phenomenon and find better correlated metric, we introduce the first theoretical framework for analyzing model merging through the lens of rate-distortion theory~\cite{berger2003rate}.
Under the Locally Modified Components (LMC) assumption, we prove a fundamental bound relating merging distortion to the geometry of hidden representations.
Our dimension-dependent theorem establishes that for representations in $R^d$, no convex merging method can achieve distortion below $\Delta^2\cdot\frac{d}{2(d+1)}$, where $\Delta$ is the diameter of task-specific representation clusters.
Our theoretical analysis aligns with the empirical findings, where hidden-state representational incompatibility at the task level is strongly correlated with merging collapse, providing a better understanding of what drives merging degradation and actionable insights for guiding task selection in improving model merging.

In summary, our main findings and contributions are as follows:

\begin{itemize}
    \item \textbf{Task-level representational incompatibility drives merging collapse.} We demonstrate that merging failure is fundamentally determined by task compatibility rather than methodology, with certain task combinations consistently causing catastrophic performance degradation across all merging approaches. Through controlled experiments, we establish that task-level representational conflicts strongly predict merging failure, while parameter-space conflict metrics show minimal correlation, challenging conventional wisdom in model merging literature.

    \item \textbf{Dimension-aware theoretical framework.} We formalize task-level merging collapse using information theory, proving that for representations in $\mathbb{R}^d$, the minimum achievable hidden-state distortion is bounded by $\Delta^2\cdot\frac{d}{2(d+1)}$ (Theorem~\ref{thm:diameter}). This dimension-dependent bound establishes precise quantitative limits on task mergeability, connecting empirical collapse to fundamental information-theoretic constraints.

    \item \textbf{Empirical validation of theoretical explanation.} We validate our theoretical predictions through extensive experiments across diverse model architectures, tasks, and merging methods. Our results confirm that representation incompatibility statistically correlates with merging collapse, with observed merging collapse following the theoretical explanations. These findings provide a principled explanation for previously puzzling cases of merging collapse.

\end{itemize}

\section{PRELIMINARY}
\label{sec:preliminary}

The development lifecycle of LLMs~\cite{bommasani2021opportunities} typically follows a two-stage paradigm: pre-training and fine-tuning. While the pre-trained model develops a wide-ranging understanding, the representations learned during pre-training are not optimized for any specific application, necessitating further fine-tuning on task-specific data with targeted objectives.

Let $M_{\theta}$ denote a Foundation Model (FM) with trainable parameters $\theta$, which can represent either the entire model for full fine-tuning or a subset of parameters for parameter-efficient fine-tuning techniques~\cite{hu2022lora, liu2022few, houlsby2019parameter}. 
In collaborative FM development environments, multiple teams often independently fine-tune the same base model for different capabilities or on different data distributions ~\cite{touvron2023llama}, resulting in a collection of specialized models.
Consider a set of tasks $\mathcal{T} = \{t_1, t_2, ..., t_n\}$. 
The based FM $M_{\theta_0}$ is fine-tuned on each task $t_i$, resulting in fine-tuned model $M_{\theta_i}$ with updated parameters $\theta_i$. 
The distributed development creates the need for effective approaches to consolidate these independently evolved models.

\textbf{Model merging.} 
Building upon the independent fine-tuning efforts described above, there arises a crucial need to integrate the strengths of multiple specialized models.
Model merging~\cite{wortsman2022model,zhang2019malmem,ilharco2023editingmodelstaskarithmetic,chitale2023taskarithmeticloracontinual,ortizjimenez2023taskarithmetictangentspace,yadav2024ties,yu2024language,yang2024model} refers to the process of combining multiple independently fine-tuned models with their parameter updates into a unified model that preserves the capabilities of its constituents.
Given the updated parameters $\{\theta_1,...,\theta_n\}$ obtained from fine-tuning the common base model $M_{\theta_0}$, the goal of model merging is to produce a consolidated model $M_{\theta_{merged}}$ that preserves the task-specific knowledge from each fine-tuned model. 

In order to achieve merging task-specific knowledge, many existing works~\cite{hendel2023context, hojel2024finding, ilharco2023editingmodelstaskarithmetic,chitale2023taskarithmeticloracontinual,ortizjimenez2023taskarithmetictangentspace} leverage \textit{Task vectors}, which provides a concise representation of the model update resulting from fine-tuning on a specific task.
Formally, let $\theta_{0} \in \mathbb{R}^d$ denote the pre-trained model weights, and $\theta_t \in \mathbb{R}^d$ the weights after fine-tuning on task $t$.
The task vector is defined as $\tau_t = \theta_t - \theta_0$, capturing both the direction and magnitude of parameter changes induced by fine-tuning.
In other words, $\tau_t$ encodes the element-wise modifications made to the model parameters for adapting to task $t$.  
Furthermore, task vectors enable a straightforward approach to model merging: by averaging the task vectors from $n$ tasks, one can construct a merged model as $\theta_{merged} = \theta_0 + (\tau_1 + \tau_2 + \dots + \tau_n)/n$.

\textbf{Parameter update conflicts.}
When merging models fine-tuned on different tasks, conflicts often arise within the parameter updates~\cite{yadav2024ties,yang2024model}.
A direct and intuitive form of conflict occurs when the elements in task vectors for different tasks have opposite signs, implying that the optimal parameter update for one task may be detrimental to another~\cite{yadav2024ties}.
However, even in the absence of sign differences, significant disparities in the magnitudes of updates can also lead to suboptimal performance in the merged model. 

Based on established literature~\cite{huang2024emr,yadav2024ties,ilharco2023editingmodelstaskarithmetic}, we focus on these four metrics to provide a comprehensive view of parameter update conflicts.

\begin{itemize}[left=8pt, itemindent=8pt, labelsep=0pt, parsep=0pt]
    \item \textbf{Parameter magnitude change ratio} is the ratio between the sum of the difference of each parameter of the same position and the sum of two models' task vectors' magnitude.
    
    \item \textbf{Parameter sign change ratio} is the ratio between the number of positions where two models' task vector not sharing a same sign and the total number of parameters positions.
    
    \item \textbf{Conflicting parameter magnitude change ratio} is the ratio between the difference of parameters at positions where two models' task vector not sharing a same sign and the sum of two models' task vectors' magnitude.

    \item \textbf{Average cosine similarity between pair's task vectors} stands for the average cosine similarity between two model's parameter vectors.

\end{itemize}

\textbf{Merging collapse}. Similar to model collapse training on synthesis data~\cite{shumailov2024ai},
we introduce the concept of merging collapse as phenomenon when distinct fine-tuned models cannot be successfully combined by a given model merging technique while preserving their original capabilities. 
Formally, let $P(\theta_i, T_i)$ denote the performance of model $M_{\theta_i}$ on task $T_i$, we quantify merging collapse using the \textbf{merging loss} $-100 \%\leq L(T_i) \leq 0 \%$ on each task $T_i$:
\begin{equation}\label{eqn::merge_loss}
    L(T_i) = (\frac{\text{P}(\theta_{\text{merged}}, T_i)}{\text{P}(\theta_i, T_i)}-1) \times 100 \%
\end{equation}
As for multiple tasks, we further define the average merging loss across all tasks as the arithmetic mean.

In the analysis that follows, we omit the percentage sign ``\%'' for simplicity.

\section{Empirical Investigation of Task-Level Model-Merging Collapse}\label{sec::empirical}
While previous work has demonstrated that model merging can effectively combine knowledge from multiple fine-tuned models, these studies typically focus on specific merging techniques~\cite{yadav2024ties,yu2024language} or specific model architecture~\cite{yadav2024matters}. 
A comprehensive understanding of when and why model merging succeeds or fail across different conditions remains an open challenge. 
To bridge the gap, we conduct a comprehensive empirical study answering the following research questions:

\begin{itemize}
    \item \textbf{RQ1:} How consistently does merging collapse occur across different merging techniques?
    \item \textbf{RQ2:} To what extent is merging collapse method-dependent or task-dependent? Do different merging techniques exhibit similar collapse patterns when confronted with the same task combinations?
    \item \textbf{RQ3:} What factor best correlates with merging compatibility? In particular, how do representation-space metrics (capturing task-level conflicts) compare to parameter-space metrics (measuring weight update conflicts) in explaining merging collapse?
\end{itemize}

\subsection{Study Setup}
\label{sec::setup}
\noindent\textbf{Models and datasets.} 
We randomly select 64 checkpoints from Lots-of-LoRAs Collection~\cite{bruel2024compress} to cover a diverse range of tasks, which are finetuned with LoRA~\cite{hu2022lora} from Mistral-7B~\cite{jiang2023mistral7b}.
We also choose eight tasks COLA, MNLI, MRPC, QNLI, QQP, RTE, SST-2 and WNLI from the GLUE dataset~\cite{wang2018glue}, which are widely used by previous work~\cite{yadav2024ties,lu2024twin} for evaluating model merging techniques.
We fine-tune \lmNum{} models including Llama3.2-3B~\cite{grattafiori2024llama3herdmodels}, Llama3.1-8B~\cite{grattafiori2024llama3herdmodels},
Qwen2.5-3B, 7B, and 14B~\cite{yang2024qwen2}, T5-Base, T5-Large, and T5-XL~\cite{raffel2023exploringlimitstransferlearning}.
These models cover different scales (from 300M to 14B), architectures (decoder-only and encoder-decoder), and training approaches (instruction tuning and task-specific fine-tuning) to improve the representativeness of our findings. 
We fine-tune these models on the eight GLUE tasks, resulting in 64 model checkpoints.

\noindent\textbf{Model merging techniques.}
We experiment with five state-of-the-art model merging techniques including \textbf{Linear Averaging (LA)}~\cite{zhang2019malmem,wortsman2022model}, \textbf{Task Arithmetic (TA)}~\cite{ilharco2023editingmodelstaskarithmetic,chitale2023taskarithmeticloracontinual,ortizjimenez2023taskarithmetictangentspace}, \textbf{TIES}~\cite{yadav2024ties}, \textbf{DARE}~\cite{yu2024language}, and \textbf{SLERP}~\cite{freeden1981spherical,goddard-etal-2024-arcees} implemented in MergeKit~\cite{goddard-etal-2024-arcees}.

\noindent\textbf{Merging settings.}
For GLUE tasks, we merge every 8 fine-tuned checkpoints from the same base model with all the five model merging techniques.
For model checkpoints from Lots-of-LoRAs, we generate 25 task groups (a)\texttildelow(y)of merging tasks, with each group containing 8 randomly selected checkpoints. We reproduce linear averaging, task arithmetic and TIES on these checkpoint groups. 

\noindent\textbf{Evaluation metrics.} 
In merging experiments of Lots-of-LoRAs, we measure the performance with rougeL~\cite{lin2004rouge} as Lots-of-LoRAs mentioned. In merging experiments of checkpoints fine-tuned on GLUE dataset, we measure the performance by classification accuracy.
We use the merging loss defined by Equation~\ref{eqn::merge_loss} in Section~\ref{sec:preliminary} to quantitatively measure the model mergeability.

\subsection{RQ1: Model Merging Collapse}

First, we conduct comprehensive experiments on GLUE tasks. Table~\ref{table::rq2::main::nlp} presents the merging loss range of existing merging techniques on GLUE tasks from which we have two major observations.

\noindent\textbf{All models suffer from severe mergeability issues in model merging.}
Regardless of model architecture, model size, or merging technique, every model examined exhibits substantial merging losses when combining multiple models simultaneously. 
\begin{table}[t]
\centering
\resizebox{\columnwidth}{!}{
\begin{threeparttable}
\caption{Merging Losses on GLUE tasks across Different Merging Techniques and Models.}
\label{table::rq2::main::nlp}
\begin{tabular}{c|ccccc}
\toprule
\multirow{2}{*}{\textbf{Model}} & \multicolumn{5}{c}{\textbf{Merging Techniques}} \\
\cline{2-6}
& \textbf{\rule{0pt}{2.5ex}LA} & \textbf{TA} & \textbf{TIES} & \textbf{DARE} & \textbf{SLERP} \\
\midrule
 Llama3.2-3B  & 17.3 $\pm$ 7.5 & 17.6 $\pm$ 7.8 & 21.8 $\pm$ 16.9 & 57.0 $\pm$ 27.5 & 17.5 $\pm$ 8.0 \\
 Llama3.1-8B  & 22.0 $\pm$ 22.2 & 22.0 $\pm$ 22.2 & 31.7 $\pm$ 19.9 & 66.1 $\pm$ 27.9 & 22.2 $\pm$ 23.0 \\
  Qwen2.5-3B  & 13.7 $\pm$ 15.1 & 13.6 $\pm$ 15.2 & 13.3 $\pm$ 13.0 & 27.1 $\pm$ 16.8 & 13.1 $\pm$ 14.7 \\
  Qwen2.5-7B  & 17.4 $\pm$ 23.3 & 17.0 $\pm$ 22.2 & 26.1 $\pm$ 24.3 & 62.6 $\pm$ 24.6 & 17.6 $\pm$ 22.6 \\
 Qwen2.5-14B  & 19.0 $\pm$ 26.3 & 19.1 $\pm$ 26.9 & 18.6 $\pm$ 25.7 & 31.9 $\pm$ 27.8 & 18.9 $\pm$ 26.3 \\
\midrule
T5-Base  & 29.0 $\pm$ 19.5 & 29.5 $\pm$ 20.5 & 31.6 $\pm$ 21.4 & 44.5 $\pm$ 13.1 & 27.7 $\pm$ 22.5 \\
 T5-Large  & 26.3 $\pm$ 24.9 & 34.4 $\pm$ 27.0 & 26.8 $\pm$ 23.5 & 28.2 $\pm$ 27.2 & 28.2 $\pm$ 27.2 \\
T5-XL  & 20.0 $\pm$ 18.8 & 19.4 $\pm$ 18.2 & 19.3 $\pm$ 18.7 & 21.2 $\pm$ 17.3 & 25.7 $\pm$ 23.3 \\
\bottomrule
\end{tabular}
\end{threeparttable}}
\end{table}

Even the best-performing combinations show double-digit merging losses, with values reaching to -32.8 across different configurations for GLUE tasks. This universal degradation demonstrates that current merging techniques fundamentally break down when scaled to practical multi-model scenarios that more closely resemble real-world deployment needs.

\noindent\textbf{No model merging technique overcomes the mergeability issue in model merging.}
While certain techniques perform marginally better in specific contexts, none successfully mitigates the fundamental mergeability problem at scale. This suggests that mergeability limitations represent an inherent challenge to scaled model merging rather than merely a technical limitation of current approaches.
The DARE technique exhibits a steepest increase in loss, which is found out to be due to its failure on certain task as previous work~\cite{deng2025dareextremerevisitingdeltaparameter} shows. 
Other techniques like LA, TA, TIES, and SLERP show more moderate but still significant loss, typically ranging from 10-25\%.

To further validate the generalization of model-merging collapse across different models and merging techniques, we conduct comprehensive experiments on Lots-of-LoRAs task groups. Out of the 25 Lots-of-LoRAs task groups we merge, 2/3 of them experience a worst performance loss of more than 30\% when merged, and only one group maintains performance within 10\% of the original models, demonstrating that severe merging collapse is the norm rather than the exception.

\noindent\fcolorbox{black}{gray!20}{
\begin{minipage}{\dimexpr\linewidth-2\fboxrule-2\fboxsep\relax}
\Answer{RQ1}{Model-merging collapse do exist for all merging techniques. Even the best-performing merging techniques show significant degradation. This suggests mergeability limitations are inherent to scaled model merging rather than technical shortcomings of specific technique.}
\end{minipage}
}

\subsection{RQ2: Method vs. Task Dependence in Merging Collapse}

\begin{table*}[t]
\footnotesize
\centering
\resizebox{2\columnwidth}{!}{
\begin{threeparttable}
\caption{Effectiveness of different merging methods on NLP Tasks with Qwen2.5 Models.}
\label{table::rq3::main}
\setlength{\tabcolsep}{3.4pt}
\begin{tabular}{l|l|cccccccccccccccc}
\toprule
\multirow{3}{*}{\textbf{Model}} & \multirow{3}{*}{\textbf{Merging}} & \multicolumn{16}{c}{\textbf{NLP Tasks}} \\
\cline{3-18}
& & \multicolumn{2}{c}{\textbf{COLA}} & \multicolumn{2}{c}{\textbf{MNLI}} & \multicolumn{2}{c}{\textbf{MRPC}} & \multicolumn{2}{c}{\textbf{QNLI}} & \multicolumn{2}{c}{\textbf{QQP}} & \multicolumn{2}{c}{\textbf{RTE}} & \multicolumn{2}{c}{\textbf{SST-2}} & \multicolumn{2}{c}{\textbf{WNLI}} \\
\cline{3-18}
& & \textbf{\rule{0pt}{2.5ex}FT} & \textbf{$M(\Delta)$} & \textbf{FT} & \textbf{$M(\Delta)$} & \textbf{FT} & \textbf{$M(\Delta)$} & \textbf{FT} & \textbf{$M(\Delta)$} & \textbf{FT} & \textbf{$M(\Delta)$} & \textbf{FT} & \textbf{$M(\Delta)$} & \textbf{FT} & \textbf{$M(\Delta)$} & \textbf{FT} & \textbf{$M(\Delta)$} \\
\midrule
\multirow{6}{*}{3B} & LA & 82.7 & 81.0 (\textcolor{red}{-2.1}) & 88.4 & 86.2 (\textcolor{red}{-2.4}) & 89.5 & 75.5 (\textcolor{red}{-15.6}) & 92.6 & 85.0 (\textcolor{red}{-8.2}) & 84.2 & 62.6 (\textcolor{red}{-25.7}) & 90.3 & 85.2 (\textcolor{red}{-5.6}) & 95.4 & 93.7 (\textcolor{red}{-1.8}) & 73.2 & 38.0 (\textcolor{red}{-48.1}) \\
 & TA & 82.7 & 81.2 (\textcolor{red}{-1.9}) & 88.4 & 86.2 (\textcolor{red}{-2.5}) & 89.5 & 75.5 (\textcolor{red}{-15.6}) & 92.6 & 85.1 (\textcolor{red}{-8.1}) & 84.2 & 62.6 (\textcolor{red}{-25.7}) & 90.3 & 85.6 (\textcolor{red}{-5.2}) & 95.4 & 93.6 (\textcolor{red}{-1.9}) & 73.2 & 38.0 (\textcolor{red}{-48.1}) \\
 & TIES & 82.7 & 79.2 (\textcolor{red}{-4.3}) & 88.4 & 87.8 (\textcolor{red}{-0.7}) & 89.5 & 83.6 (\textcolor{red}{-6.6}) & 92.6 & 67.8 (\textcolor{red}{-26.9}) & 84.2 & 63.2 (\textcolor{red}{-24.9}) & 90.3 & 87.7 (\textcolor{red}{-2.8}) & 95.4 & 91.9 (\textcolor{red}{-3.7}) & 73.2 & 46.5 (\textcolor{red}{-36.5}) \\
 & DARE & 82.7 & 70.4 (\textcolor{red}{-14.9}) & 88.4 & 35.4 (\textcolor{red}{-60.0}) & 89.5 & 82.6 (\textcolor{red}{-7.7}) & 92.6 & 50.6 (\textcolor{red}{-45.4}) & 84.2 & 63.2 (\textcolor{red}{-24.9}) & 90.3 & 80.9 (\textcolor{red}{-10.4}) & 95.4 & 73.7 (\textcolor{red}{-22.7}) & 73.2 & 50.7 (\textcolor{red}{-30.8}) \\
 & SLERP & 82.7 & 81.7 (\textcolor{red}{-1.3}) & 88.4 & 86.2 (\textcolor{red}{-2.4}) & 89.5 & 76.0 (\textcolor{red}{-15.1}) & 92.6 & 85.4 (\textcolor{red}{-7.8}) & 84.2 & 62.6 (\textcolor{red}{-25.6}) & 90.3 & 85.9 (\textcolor{red}{-4.8}) & 95.4 & 93.7 (\textcolor{red}{-1.8}) & 73.2 & 39.4 (\textcolor{red}{-46.2}) \\

\midrule
\multirow{6}{*}{7B} & LA & 86.1 & 82.5 (\textcolor{red}{-4.2}) & 90.8 & 86.1 (\textcolor{red}{-5.1}) & 88.0 & 74.0 (\textcolor{red}{-15.9}) & 95.3 & 88.3 (\textcolor{red}{-7.4}) & 88.0 & 63.2 (\textcolor{red}{-28.2}) & 89.9 & 88.1 (\textcolor{red}{-2.0}) & 96.0 & 94.5 (\textcolor{red}{-1.6}) & 78.9 & 19.7 (\textcolor{red}{-75.0}) \\
 & TA & 86.1 & 82.9 (\textcolor{red}{-3.7}) & 90.8 & 86.1 (\textcolor{red}{-5.1}) & 88.0 & 74.0 (\textcolor{red}{-15.9}) & 95.3 & 88.1 (\textcolor{red}{-7.6}) & 88.0 & 63.2 (\textcolor{red}{-28.2}) & 89.9 & 87.4 (\textcolor{red}{-2.8}) & 96.0 & 94.4 (\textcolor{red}{-1.7}) & 78.9 & 22.5 (\textcolor{red}{-71.4}) \\
 & TIES & 86.1 & 62.2 (\textcolor{red}{-27.7}) & 90.8 & 88.2 (\textcolor{red}{-2.8}) & 88.0 & 55.4 (\textcolor{red}{-37.0}) & 95.3 & 64.4 (\textcolor{red}{-32.5}) & 88.0 & 63.2 (\textcolor{red}{-28.2}) & 89.9 & 89.5 (\textcolor{red}{-0.4}) & 96.0 & 94.4 (\textcolor{red}{-1.7}) & 78.9 & 16.9 (\textcolor{red}{-78.6}) \\
 & DARE & 86.1 & 0.0 (\textcolor{red}{-100.0}) & 90.8 & 27.0 (\textcolor{red}{-70.2}) & 88.0 & 31.6 (\textcolor{red}{-64.1}) & 95.3 & 50.5 (\textcolor{red}{-47.0}) & 88.0 & 63.2 (\textcolor{red}{-28.2}) & 89.9 & 47.7 (\textcolor{red}{-47.0}) & 96.0 & 0.0 (\textcolor{red}{-100.0}) & 78.9 & 43.7 (\textcolor{red}{-44.6}) \\
 & SLERP & 86.1 & 82.6 (\textcolor{red}{-4.0}) & 90.8 & 86.2 (\textcolor{red}{-5.0}) & 88.0 & 73.3 (\textcolor{red}{-16.7}) & 95.3 & 88.4 (\textcolor{red}{-7.3}) & 88.0 & 63.2 (\textcolor{red}{-28.2}) & 89.9 & 85.9 (\textcolor{red}{-4.4}) & 96.0 & 94.4 (\textcolor{red}{-1.7}) & 78.9 & 21.1 (\textcolor{red}{-73.2}) \\

\midrule
\multirow{6}{*}{14B} & LA & 88.1 & 81.5 (\textcolor{red}{-7.5}) & 91.9 & 90.1 (\textcolor{red}{-1.9}) & 90.4 & 78.2 (\textcolor{red}{-13.6}) & 95.9 & 89.0 (\textcolor{red}{-7.2}) & 88.7 & 62.4 (\textcolor{red}{-29.7}) & 93.9 & 88.8 (\textcolor{red}{-5.4}) & 97.1 & 95.1 (\textcolor{red}{-2.1}) & 84.5 & 12.7 (\textcolor{red}{-85.0}) \\
 & TA & 88.1 & 81.4 (\textcolor{red}{-7.6}) & 91.9 & 90.2 (\textcolor{red}{-1.9}) & 90.4 & 78.7 (\textcolor{red}{-13.0}) & 95.9 & 89.0 (\textcolor{red}{-7.3}) & 88.7 & 62.4 (\textcolor{red}{-29.7}) & 93.9 & 89.2 (\textcolor{red}{-5.0}) & 97.1 & 95.2 (\textcolor{red}{-2.0}) & 84.5 & 11.3 (\textcolor{red}{-86.7}) \\
 & TIES & 88.1 & 81.6 (\textcolor{red}{-7.4}) & 91.9 & 88.2 (\textcolor{red}{-4.1}) & 90.4 & 81.6 (\textcolor{red}{-9.8}) & 95.9 & 81.5 (\textcolor{red}{-15.1}) & 88.7 & 70.8 (\textcolor{red}{-20.2}) & 93.9 & 88.1 (\textcolor{red}{-6.2}) & 97.1 & 96.1 (\textcolor{red}{-1.1}) & 84.5 & 12.7 (\textcolor{red}{-85.0}) \\
 & DARE & 88.1 & 28.4 (\textcolor{red}{-67.8}) & 91.9 & 57.8 (\textcolor{red}{-37.1}) & 90.4 & 85.5 (\textcolor{red}{-5.4}) & 95.9 & 50.5 (\textcolor{red}{-47.3}) & 88.7 & 76.0 (\textcolor{red}{-14.3}) & 93.9 & 92.1 (\textcolor{red}{-1.9}) & 97.1 & 92.4 (\textcolor{red}{-4.8}) & 84.5 & 19.7 (\textcolor{red}{-76.7}) \\
 & SLERP & 88.1 & 81.4 (\textcolor{red}{-7.6}) & 91.9 & 90.2 (\textcolor{red}{-1.9}) & 90.4 & 78.9 (\textcolor{red}{-12.7}) & 95.9 & 89.2 (\textcolor{red}{-7.0}) & 88.7 & 62.4 (\textcolor{red}{-29.6}) & 93.9 & 88.4 (\textcolor{red}{-5.8}) & 97.1 & 95.2 (\textcolor{red}{-2.0}) & 84.5 & 12.7 (\textcolor{red}{-85.0}) \\

\bottomrule
\end{tabular}

\begin{tablenotes}
\small
\item[*] FT: Fine-tuned model performance; $M(\Delta)$: Performance after merging and the negative number of merging loss following Equation~\ref{eqn::merge_loss}.
\end{tablenotes}
\end{threeparttable}}
\end{table*}

\begin{table}[t]
\centering
\caption{P-values of One-way ANOVA F-tests~\cite{maxwell2017designing} for Task-level and Merging-technique-level Effects. $p < 0.05$ indicates statistical significance.}
\label{table::rq1::pvalue}
\setlength{\tabcolsep}{9pt}
\begin{tabular}{l|cc}
\toprule
\textbf{Task Setting} & \textbf{Merging Tech.} &  \textbf{Task}\\
\midrule
Glue &$0.575$ & $2.357\times 10^{-36}$
\\
Lots-of-LoRAs & $0.987$ & 
$ 1.699 \times 10^{-7}$
 \\
\bottomrule
\end{tabular}
\end{table}

To investigate whether the revealed merging collapse in RQ1 is due to the imperfection of existing merging methods or due to the inherent incompatibility of model checkpoints, in Table~\ref{table::rq3::main} we present a detailed merging results with Qwen2.5 models. It can be observed from Table~\ref{table::rq3::main} that there exists some tasks, such as MRPC and WNLI, seem to suffer severe merging loss across all merging techniques, indicating the choice of merging technique has little impact over performance in certain tasks.

We further conduct statistic tests on how merging technique and merged task are correlated to merging loss. The statistical evidence presented in Table~\ref{table::rq1::pvalue} provides compelling support for task-level incompatibility as the primary cause of merging collapse. Both tables show a stark contrast in 
p-values between merging technique-level and task-level effects. 
These consistent findings across different datasets and model architectures strongly suggest that merging collapse stems primarily from fundamental task incompatibilities rather than limitations in merging techniques. The choice of merging technique appears to have minimal impact on performance compared to the inherent compatibility between task representations, confirming our hypothesis that certain tasks face intrinsic barriers to successful merging regardless of the techniques employed.

\noindent\fcolorbox{black}{gray!20}{
\begin{minipage}{\dimexpr\linewidth-2\fboxrule-2\fboxsep\relax}
\Answer{RQ2}{Statistical analysis reveals that merging collapse is primarily task-dependent rather than method-dependent. The highly significant task-level effects contrasted with non-significant merging technique-level effects across both datasets demonstrate that inherent task incompatibilities, not methodological limitations, are the key to merging collapse.}
\end{minipage}
}

\subsection{RQ3: Correlative Factors for Merging Collapse}

\subsubsection{Theoretical Explanation}
Our empirical results establish a clear pattern, where merging collapse is not determined methodological choices but inherent task incompatibility. 

To shed light on these empirical results with theoretical analysis, we develop a theoretical framework that formalizes the relationship between representational incompatibility and merging collapse.

By modeling the merging process through the lens of rate-distortion theory~\cite{berger2003rate}, we can establish lower bounds on the minimum distortion achievable when combining representations from different tasks. 
Under the assumption of LMC~\cite{frankle2020linear}, which is widely observed in pre-training and fine-tuning paradigm~\cite{zhou2024emergence}, we prove the following theorem.

\begin{theorem}[Hidden–State Diameter Controls Mergeability]
\label{thm:diameter}
Let $\{\,\theta_i\}_{i=1}^N\subset\mathbb R^{p}$ be $N$ fine-tuned minima of the same
base network $F(\cdot;\theta)$, and let $h(x;\theta)\in\mathbb R^{d}$ be a fixed
hidden layer.  Assume {\em linear mode connectivity} (LMC): every convex
combination $\sum_i\alpha_i\theta_i$ ($\alpha_i\!\ge0,\sum\alpha_i=1$) attains
the same training loss $\le\varepsilon$.  Denote
\[
d^2(i,j)\;=\;\E_X\bigl\|h(X;\theta_i)-h(X;\theta_j)\bigr\|_2^2,
\Delta\;=\;\max_{i,j} d(i,j)
\]
and for any candidate merged model $\hat\theta$ define the worst-case hidden-state
distortion
$
\delta_{\max}(\hat\theta)=\max_i\E_X\!\|h(X;\hat\theta)-h(X;\theta_i)\|_2^{2}.
$

(i) (Achievability) There exists a convex merge
$\bar\theta=\sum_{i} \alpha_i\theta_i$ such that
$\delta_{\max}(\bar\theta)\leq\frac{d}{2(d+1)}\Delta^2$.

(ii) (Converse) Every merge satisfies $\delta_{\max}(\hat\theta)\ge\frac14\Delta^2$; hence
$\frac14\Delta^2$ is the minimum attainable distortion $D^\star$.

(iii) Viewing $I\!\sim\!\mathrm{Unif}\{1,\dots,N\}$ and
$Y=h(X;\theta_I)$ as the source, the Shannon rate–distortion curve obeys
$R(D)=0$ iff $D\ge D^\star$, and $R(D)\ge\log_2 N$ for $D<D^\star$.

Consequently, under LMC the single scalar $\Delta$ completely
characterises (a) whether the experts are mergeable within a budget
$D_\text{budget}$ and (b) the zero-rate point of the
rate–distortion function.
\end{theorem}

\begin{proof}[Sketch]
Since LMC implies {\em linearity of hidden states in parameter space},
the hidden representation of any convex merge lies in the convex hull of
$\{h(\cdot;\theta_i)\}$.
The minimum enclosing ball of a finite set in
$\mathbb R^{d}$ has radius less than $\frac{d}{2(d+1)}$ of its diameter (Jung's Theorem~\cite{jung1901ueber}).  
Details are in
our Appendix~\ref{app:diameter-proof}. \qedhere
\end{proof}

\subsubsection{Hidden-state Distance Similarity.}  Theorem~\ref{thm:diameter} implies that the representations' distances between different models fine-tuned on different tasks can be an effective metric to capture the task-level merging collapse. 
Inspired by theorem~\ref{thm:diameter}, we propose Hidden-state Distance Similarity which is calculated based on the average L2-distance between the hidden states of different models processing a same set of input, referred as $d_{i,j}$.
    Assuming we calculate hidden-state distance similarity over a group of models $\{1,2,...,n\}$ and their hidden states' L2-distances $D = d_{1,2},d_{1,3},...d_{n-1,n}$, hidden-state distance similarity is calculated through:
    \begin{equation}
        HiddenSim(i,j) = 
        \begin{cases}
        1 & i=j \\
        
        \frac{max(D)-d_{i,j}}{max(D) - min(D)}
        & i\ne j
        \end{cases}
    \end{equation}
    , where $HiddenSim$ stands for the Hidden-state distance similarity.
    In our experiments measuring hidden state distances, we draw $k=5$ datapoints from every task's dataset and compose them into a validation dataset. 
By running models on the validation dataset, we compare their hidden states of the last layer by calculating the normalized L2 distance averaged across each datapoint.
    
\subsubsection{Empirical results on Correlation when merging pairs of models}
\begin{table*}[t]
\centering
\caption{P-values of Pearson correlation coefficient~\cite{benesty2009pearson} on how conflict metrics are correlated to merging loss with different merging techniques when merging two model checkpoints. $p < 0.05$ indicates statistical significance.}
\label{table::rq1::4metric}
\setlength{\tabcolsep}{4pt}

\begin{tabular}{c|ccccc}
\toprule
\multirow{2}{*}{\textbf{Conflict Metrics}} & \multicolumn{5}{c}{\textbf{Merging Techniques}} \\
\cmidrule{2-6}
& \textbf{TIES} & \textbf{LA} & \textbf{DARE} & \textbf{SLERP} & \textbf{TA} \\
\midrule
Parameter Magnitude Change Ratio & 0.192 & 0.098 & 0.834 & 0.085 & 0.186 \\
Parameter Sign Change Ratio & 0.379 & 0.460 & 0.882 & 0.408 & 0.659 \\
Conflicting Parameter Magnitude Change Ratio & 0.170 & 0.170 & 0.979 & 0.156 & 0.235 \\
Average Cosine Similarity & 0.262 & 0.459 & 0.635 & 0.444 & 0.296 \\
\midrule
\textbf{Hidden State Distance Similarity }& \textbf{0.006} & \textbf{0.001} & 0.145 & \textbf{0.001} & \textbf{0.006} \\
\bottomrule
\end{tabular}
\end{table*}

To validate our theoretical analysis with empirical results, we first investigate the correlation between conflict metrics and merging collapse. 

\noindent\textbf{Experiment Setting.} 
As the conflict metrics are defined between two models, we perform additional experiment merging all the pairs ($C_8^2=28$ pairs) of Qwen2.5-3B checkpoints on GLUE dataset across five merging techniques. Though merging two models is much easier than merging eight, merging collapse persists in certain task combinations.

\noindent\textbf{Results.} Contrary to intuitive expectations, the widely adopted parameter update conflicts we introduced in Section~\ref{sec:preliminary} exhibit remarkably weak correlations with merging collapse across all four parameter update conflict metrics. This absence of correlation of existing parameter update conflict metrics is particularly evident in Table~\ref{table::rq1::4metric}, where we calculate the p-values of Pearson correlation coefficient~\cite{benesty2009pearson} to analyze how existing parameter conflict metrics are correlated to merging loss. With all p-values $> 0.05$, though we can't sufficiently conclude that existing parameter conflict metrics has completely no correlation with merging loss, it is inarguable that they perform poorly to capture the underlying mechanisms driving merging collapse, further indicating that the phenomenon likely emerges from more complex, task-related incompatibilities rather than straightforward parameter update disagreements.

Our hidden state distance similarity metric, on the other hand, shows a strong correlation with merging collapse. Except for DARE, which is found to have unstable merging performance like discussed in RQ1, all p-values of Pearson correlation coefficient are much less than the threshold of $p=0.05$. Even for DARE, our metric keeps a drastic advantage over existing metrics in correlation. Our metric corresponds significantly better to the actual merging performance than existing parameter update conflicts.

\subsubsection{Empirical results on hiddensim for scaled merging.}

We further extend our theory explanation to the situation where more models are merged together.

\noindent\textbf{Experiment Setting.}  We perform our empirical investigation on four 8-task groups: one task group consisting of 8 Qwen2.5-3B checkpoints fine-tuned on GLUE tasks, and three task groups (a), (b), (c), each containing of 8 Lots-of-LoRAs tasks.

\noindent\textbf{Visualization.} 
Figure~\ref{fig:heatmap_glue} demonstrates a strong correlation between hidden state similarity and merging performance in merging eight Qwen2.5-3B checkpoints in GLUE tasks, where most tasks show high similarity scores with other tasks, while certain tasks like QQP and WNLI exhibit significantly lower scores with other tasks, explaining their poor merging outcomes in a certain way. This pattern extends to LoRA task groups, where we adopt similar experiments on three Lots-of-LoRAs task groups (a), (b), (c) and find out similar correlation between lower similarity score and awful merging performance or even merging collapse.

\noindent\textbf{Quantitative Analysis.} To better understand how hidden state similarity score is correlated to merging performance and conduct quantitative analysis, we define the \textit{Merging Difficulty Score (MDS)} for each task $i$ as the reciprocal of average representational similarity scores with other tasks ($j\neq i$) : 
\begin{equation}
MDS_i = \frac{1}{\text{Average}_{j\neq i}(HiddenSim(i,j))}
\end{equation} 
This metric, inspired by the resistance in physics, provides a non-linear measure of similarity scores amplifying the impact of low similarities where higher MDS indicates worse mergeability (higher merging loss), i.e., greater resistance to merging.

In Table~\ref{mds1}\texttildelow~\ref{mds4}, we present the calculated MDS value and bold the columns with the most severe merging collapse across all merging techniques. Without exception, all tasks with merging collapse have a significantly larger MDS. A statistical analysis is also shown in Table~\ref{table::mds::correlation}, which shows the p-values of Pearson correlation coefficients between every task's MDS and \textit{best} merging loss across different merging techniques. Among all four eight-task groups, all p-values are lower than the threshold $p=0.05$. Results in these tables support our expectation of MDS that higher MDS indicates worse mergeability, providing further empirical evidence of our theory.

  \begin{minipage}{1\linewidth}
    \centering
 \includegraphics[width=1.0\linewidth]{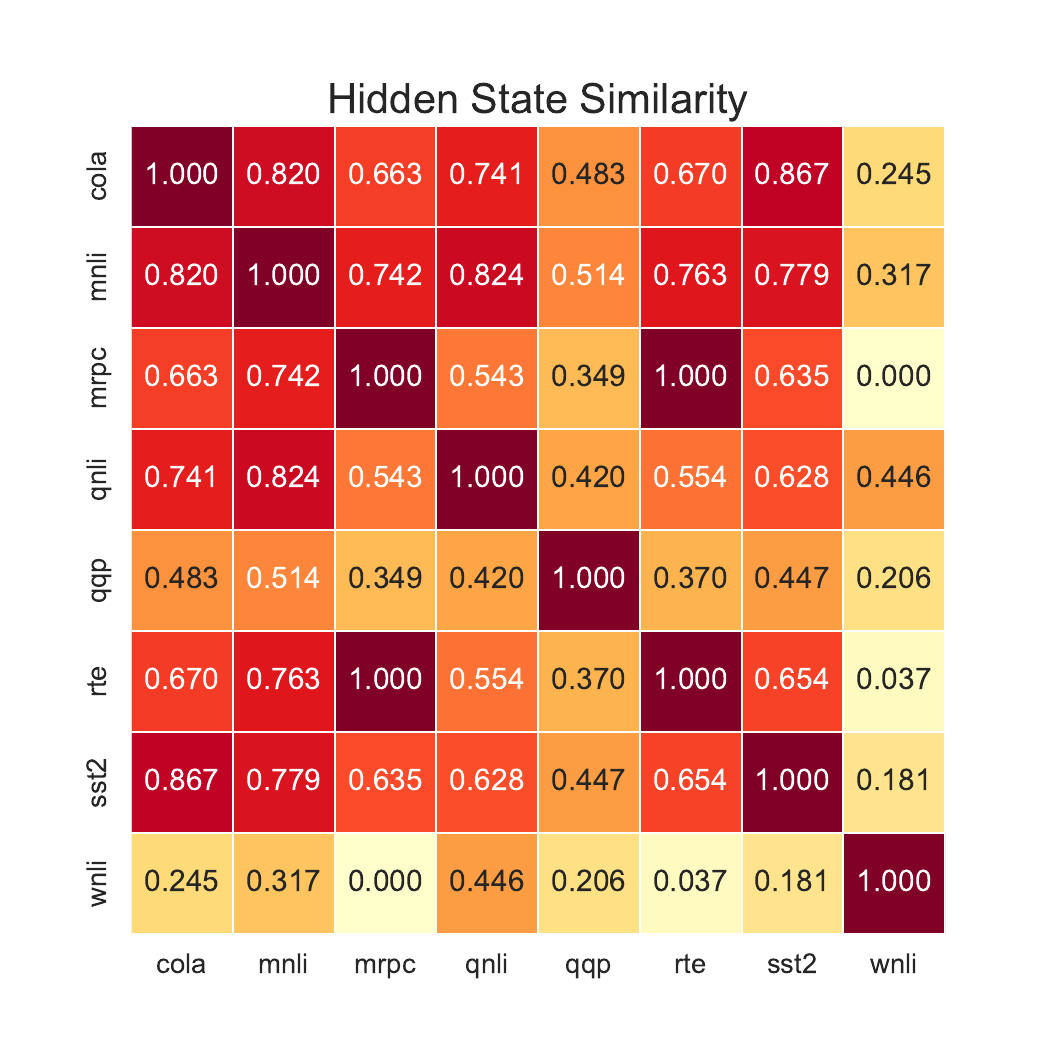}
\captionof{figure}{Heatmap of similarity score of hidden states in GLUE tasks.}
  \label{fig:heatmap_glue}
  \end{minipage}%

  \begin{minipage}{1\linewidth}
    \centering
 \captionof{table}{P-values of Pearson correlation coefficients.
on how MDS is correlated to the \textit{best} merging loss in merging Qwen2.5-3B checkpoints on GLUE and Lots-of-LoRAs task groups (a), (b) and (c) in experiments. 
$p < 0.05$ indicates statistical significance.
}
\begin{tabular}{c|c}
\toprule
\textbf{Merging tasks}  & \textbf{p-value} \\
\midrule
GLUE & \textbf{0.0032 $\ll$ \textit{0.05}} \\
group (a) & \textbf{0.0013 $\ll$ \textit{0.05}} \\
group (b) & \textbf{0.0288 $<$ \textit{0.05}} \\
group (c)  & \textbf{0.0059 $\ll$ \textit{0.05}} \\
\bottomrule
\end{tabular}
\label{table::mds::correlation}

  \end{minipage}

These findings from both intuitive figures and quantitative analysis confirm that representational compatibility at the hidden state level, rather than parameter update conflicts, is more correlated with merging success or collapse, validating our theoretical analysis that merging collapse stems from differences in model representations.

\begin{table}[t]
\footnotesize
\centering
\caption{Comparison among three task groups.}
\label{table::guide}
\setlength{\tabcolsep}{6pt}
\begin{tabular}{l|c|cc}
\toprule
 & \textbf{(a)} & \textbf{(a1)} & \textbf{(a2)} \\
\midrule
Worst Loss & -70.8 & -17.8 & -19.1 \\
\hline
Worst MDS & 5.05 & 2.52 & 2.89 \\
\bottomrule
\end{tabular}
\small
\end{table}

\noindent\textbf{Merging Tasks Selection Guiding.} Merging models with higher similarity scores could lead to better merging performance. In group (a), the 6-th task incurs significant merging collapse. Using our MDS metric as a guide, we strategically replace this task with more compatible alternatives from Lots-of-LoRA tasks and construct task gropus (a1) and (a2). Comparing their merging performance with TIES as shown in Table~\ref{table::guide}, we find that merging tasks with lower MDS do achieve lesser merging collapse, further supporting our theory explanation and demonstrating the potential usefulness of our findings.

\begin{minipage}[t]{1\linewidth}
        \centering

\resizebox{1\columnwidth}{!}{
\begin{threeparttable}
\centering
\caption{Detail data when merging 8 checkpoints from GLUE dataset.}

\label{mds1}
\setlength{\tabcolsep}{4pt}
\begin{tabular}{c|ccccccccc}
\toprule

LA & -2.1 & -2.4 & -15.6 & -8.2 & -25.7 & -5.6 & -1.8 & \textbf{-48.1} \\
TA & -1.9 & -2.5 & -15.6 & -8.1 & -25.7 & -5.2 & -1.9 & \textbf{-48.1} \\
TIES & -4.3 & -0.7 & -6.6 & -26.9 & -24.9 & -10.4 & -22.7 & \textbf{-30.8} \\
\midrule
MDS & 1.56 & 1.47 & 1.78 & 1.68 & 2.51 & 1.73 & 1.72 & \textbf{4.89} \\

\bottomrule
\end{tabular}
\end{threeparttable}}
\end{minipage}
\begin{minipage}[t]{1\linewidth}
        \centering
\resizebox{1\columnwidth}{!}{
\begin{threeparttable}
\centering
\caption{Detail data when merging group (a) tasks from Lots-of-LoRAs.}
\label{MDS2}

\setlength{\tabcolsep}{4pt}
\begin{tabular}{c|ccccccccc}
\toprule
LA & -9.7 & -6.8 & -0.6 & 17.0 & -4.9 & \textbf{-78.0} & -4.6 & -16.5  \\
TA & -17.9 & -18.7 & -3.7 & 12.8 & -6.3 & \textbf{-78.6} & -6.4 & -11.4  \\
TIES & -10.8 & -6.0 & -3.7 & -0.9 & -5.6 & \textbf{-70.8} & -3.1 & -15.5  \\
\midrule
MDS & 2.85 & 1.99 & 2.90 & 1.64 & 1.68 & \textbf{5.05} & 1.98 & 2.65 \\
\bottomrule
\end{tabular}
\end{threeparttable}}
\end{minipage}

\begin{minipage}[t]{1\linewidth}
        \centering
\resizebox{1\columnwidth}{!}{
\begin{threeparttable}
\centering
\caption{Detail data when merging group (B) tasks from Lots-of-LoRAs.}
\label{mds3}

\setlength{\tabcolsep}{4pt}
\begin{tabular}{c|ccccccccc}
\toprule

LA & 2.5 & -9.1 & -30.0 & -11.7 & -10.7 & \textbf{-75.8} & -40.2 & 0.7  \\
TA & -16.0 & -16.3 & -57.5 & -14.7 & -17.6 & \textbf{-78.9} & -44.7 & -9.9  \\
TIES & -4.1 & -6.8 & -39.1 & -9.8 & -14.5 & \textbf{-68.2} & -54.8 & -0.9  \\
\midrule
MDS &2.22 & 2.78 & 1.62 & 2.19 & 1.45 & \textbf{8.15} & 1.94 & 1.43 \\

\bottomrule
\end{tabular}
\end{threeparttable}}
\end{minipage}
\begin{minipage}[t]{1\linewidth}
        \centering
\resizebox{1\columnwidth}{!}{
\begin{threeparttable}
\centering
\caption{Detail data when merging group (C) tasks from Lots-of-LoRAs.}
\label{mds4}

\setlength{\tabcolsep}{4pt}
\begin{tabular}{c|ccccccccc}
\toprule
LA & 2.9 & -15.2 & \textbf{-62.8} & -18.6 & -13.1 & -5.8 & -8.3 & -1.0  \\
TA & -7.3 & -40.7 &\textbf{ -67.6 }& -22.3 & -10.4 & 26.3 & -35.5 & -15.7  \\
TIES & -0.1 & -32.6 & \textbf{-57.1} & -5.9 & -14.7 & 6.6 & -4.5 & -1.0  \\
\midrule
MDS & 1.45 & 1.30 & \textbf{3.75} & 1.23 & 1.70 & 1.23 & 1.24 & 1.28 \\
\bottomrule
\end{tabular}
\end{threeparttable}}
\end{minipage}

\noindent\fcolorbox{black}{gray!20}{
\begin{minipage}{\dimexpr\linewidth-2\fboxrule-2\fboxsep\relax}
\Answer{RQ3}{While parameter-level conflicts show little correlation with merging collapse across multiple conflict calculation methods, hidden state similarity strongly correlates with merging success, with higher representational compatibility corresponding to better merged performance. This confirms that merging collapse stems from fundamental differences in task representations rather than parameter disagreements, explaining why catastrophic degradation exists across different merging techniques.}
\end{minipage}
}

\section{Related Work}\label{sec::related}

\subsection{Development of LLMs}
LLMs are increasingly being deployed across a wide range of application scenarios~\cite{bommasani2021opportunities}.
These application scenarios span a spectrum of specialized domains, ranging from intelligent customer service systems requiring precise query resolution~\cite{xiaoliang2024design, pandya2023automating} to clinical decision-support tools in healthcare diagnostics~\cite{abbasian2023conversational, zhang2024llm}, from real-time risk assessment frameworks in financial operations~\cite{zhao2024revolutionizing, hasan2020current} to adaptive learning platforms for personalized educational recommendations~\cite{urdaneta2021recommendation}.
This substantial diversity in task requirements and operational contexts poses unique challenges when adapting general-purpose LLMs to domain-specific downstream tasks.
To enhance targeted capabilities of LLMs in specialized applications, fine-tuning has emerged as a predominant methodological paradigm~\cite{yang2019xlnet,brown2020language,raffel2023exploringlimitstransferlearning}, enabling the alignment of pre-trained knowledge with task-specific objectives on curated domain datasets.
The pre-training and fine-tuning paradigm makes linear mode connectivity~\cite{nagarajan2019uniform, frankle2020linear}—where solutions in parameter space can be linearly interpolated while maintaining performance—as an emergent property of fine-tuned models~\cite{qin2022exploring,zhou2024emergence}.

\subsection{Model Merging}
There is a growing body of work on post-training model merging~\cite{wortsman2022model,zhang2019malmem,ilharco2023editingmodelstaskarithmetic,chitale2023taskarithmeticloracontinual,ortizjimenez2023taskarithmetictangentspace,yadav2024ties,yu2024language}, which focuses on combining pre-trained models that have been fine-tuned on different tasks. 
Linear averaging approaches~\cite{wortsman2022model,zhang2019malmem} average parameters of multiple
fine-tuned models element-wise. 
Task arithmetic~\cite{ilharco2023editingmodelstaskarithmetic,chitale2023taskarithmeticloracontinual,ortizjimenez2023taskarithmetictangentspace} leverages task-specific weight vectors, to represent weights for specific tasks to facilitate direct manipulation of multi-task learning by adding or subtracting on task vectors.
These approaches are vulnerable to parameter update conflicts~\cite{chen2024you,kong2024activated}.
TIES~\cite{yadav2024ties} tries to reduce parameter update conflicts with heuristics such as trimming minor updates to avoid redundant updates~\cite{zbontar2021barlow}.
DARE~\cite{yu2024language} improves upon TIES by re-scaling the remaining parameters to keep the distributional stability. 
Twin-merging~\cite{lu2024twin} and AdaMerging~\cite{yang2023adamerging} relaxes the storage constraints to improve the model performance after merging. 
Despite these advances, a fundamental understanding of merge conflicts remains underdeveloped.

\section{Conclusion}\label{sec::conclusion}
In this paper, we have investigated the fundamental limitations of model merging and identified the key factors that determine merging success or failure. Our research reveals that task-level representational incompatibility is the primary driver of merging collapse, with certain task combinations consistently failing across all merging methods. This finding challenges the conventional wisdom that parameter conflicts are the main obstacle to successful merging.
We further develop a theoretical framework based on rate-distortion theory that establishes relationship between model representation distances and merging collapse, verified by empirical evidence.

\clearpage

\bibliographystyle{icml2026}
\bibliography{ref}

\clearpage

\section*{Acknowledgments}
Tao Xie is with the Key Laboratory of High Confidence Software Technologies (Peking University), Ministry of Education; School of Computer Science, Peking University, Beijing, China; Beijing Tongming Lake Information Technology Application Innovation Center; Fudan University Institute of Systems for Advanced Computing, Shanghai, China; Shanghai Institute of Systems for Open Computing, Shanghai, China.
This work was partially supported by the National Natural Science Foundation of China under Grant No. 92464301, U25A6023.
Dezhi Ran is partially supported by a Hunyuan Scholar Award.

\section{Proof of Theorem~\ref{thm:diameter}}
\label{app:diameter-proof}

\begin{lemma}[Jung's Theorem~\cite{jung1901ueber}]
\label{lem:minball}
For any finite set $\mathcal S\subset\R^{d}$ with diameter
$\operatorname{diam}(\mathcal S)=
\max_{u,v\in\mathcal S}\norm{u-v}_2=\Delta$,
there exists a point $c\in\mathrm{conv}(\mathcal S)$ such that
\[
   \max_{s\in\mathcal S}\|c-s\|_2 \;\le\; \sqrt{\frac{d}{2(d+1)}}\,\Delta.
\]

\end{lemma}
\begin{proof}
The centre of the minimum enclosing ball of $\mathcal S$
has the stated properties and lies in the convex hull
(by classic results on the smallest enclosing ball
in Euclidean space~\cite{megiddo1983linear}).
\end{proof}

\begin{proof}[Proof of Theorem~\ref{thm:diameter}]
Throughout we abbreviate $H_i(x)=h(x;\theta_i)$ for the unique input $x$.

\paragraph{Step 1: Achievability ($D^\star\le\frac{d}{2(d+1)}\,\Delta^2$).}
Fix any coefficients $\alpha=\{\alpha_i\}_{i=1}^{N}$
with $\alpha_i\ge0,\ \sum_i\alpha_i=1$.
LMC ensures that the hidden state of their convex merge satisfies
\begin{equation}\label{eq:linear-h}
h(x;\bar\theta)=h\!\Bigl(x;\sum_i\alpha_i\theta_i\Bigr)
        =\sum_i\alpha_i\,h(x;\theta_i)
        =\sum_i\alpha_i H_i(x)\;\in\mathrm{conv}\{H_i(x)\}_{i}.
\end{equation}
Apply Lemma~\ref{lem:minball} to the finite set
$\{H_i(x)\}_{i}$: choosing {\em the same} convex coefficients
$\alpha$ ensures that $h(x;\bar\theta)$ coincides with
the centre $c(x)$ of the minimum enclosing ball, hence
$\norm{h(x;\bar\theta)-H_i(x)}_2\le \sqrt{\frac{d}{2(d+1)}}\,\Delta_x$
for every $i$, where
$\Delta_x=\max_{j,k}\norm{H_j(x)-H_k(x)}_2\le\Delta$.
Taking expectations and the maximum over $i$ yields
$\delta_{\max}(\bar\theta)\le \frac{d}{2(d+1)}\,\Delta^2$.

\paragraph{Step 2: Converse ($D^\star\ge\tfrac14\,\Delta^{2}$).}
Let $\hat\theta$ be any merged model and choose
indices $(i,j)$ that realise the diameter
$d(i,j)=\Delta$.  The triangle inequality gives
\[
d(i,j) \;=\;\sqrt{\E_X \norm{H_i(X)-H_j(X)}_2^{2}}
          \;\le\;\sqrt{\delta_i(\hat\theta)}+\sqrt{\delta_j(\hat\theta)}
          \;\le\;2\sqrt{\delta_{\max}(\hat\theta)}.
\]
Squaring both sides implies
$\delta_{\max}(\hat\theta)\ge\tfrac14\,\Delta^{2}$,
proving the lower bound.

\paragraph{Step 3: Rate–distortion statement.}
Let the random pair $(X,I)$ be distributed according to
$P_X\times\mathrm{Unif}\{1,\dots,N\}$, and define
$Y=h(X;\theta_I)$.  A {\em zero-rate} encoder cannot send any
information about $I$, so the decoder must output a reconstruction
$\hat Y$ that is independent of $I$ and is therefore produced by a
single parameter vector $\hat\theta$.  The minimum achievable
mean-squared error is precisely $D^\star$.  Shannon’s converse then
asserts that $R(D)=0$ iff $D\ge D^\star$, while for $D<D^\star$ at
least $\log_2 N$ bits are necessary to transmit the expert identity,
establishing the claimed $R(D)$ behaviour.

\paragraph{Step 4: Putting the pieces together.}
Steps 1 and 2 show $\tfrac14\,\Delta^{2}\le D^\star\le\frac{d}{2(d+1)}\,\Delta^2$ and identify any convex
merge (including the uniform average)
as optimal.  Step~3 connects this constant to the fundamental
rate–distortion limit, completing the proof.
\end{proof}

\begin{corollary}[Practical mergeability test]
Given a distortion budget $D_{\text{budget}}$, a set of models is
mergeable under LMC iff their hidden–state diameter $\Delta$
satisfies $\frac{d}{2(d+1)}\,\Delta^2\le D_{\text{budget}}$. Since the hidden dimension $d$ is typically large in language models, $\frac{d}{2(d+1)}$ can be approximated by $\tfrac{1}{2}$.
\end{corollary}

\clearpage
\end{document}